\documentclass[letterpaper, 10 pt, conference]{ieeeconf}  

\IEEEoverridecommandlockouts                              

\usepackage{amsmath,amssymb}
\usepackage{mathtools}
\usepackage{bm}
\usepackage{graphicx}
\usepackage{algorithm}
\usepackage{algpseudocode}
\usepackage{booktabs}
\usepackage{xcolor}
\usepackage{comment}
\usepackage{hyperref}
\usepackage{cleveref}
\usepackage{tikz}
\usepackage{cite}
\usetikzlibrary{calc,positioning,arrows.meta,backgrounds,shapes.geometric}

\newtheorem{theorem}{Theorem}
\newtheorem{proposition}{Proposition}
\newtheorem{assumption}{Assumption}
\newtheorem{definition}{Definition}

\newcommand{\N}{\mathbb{N}}
\newcommand{\col}{\mathrm{col}}
\newcommand{\Ball}[2]{B_{#1}\!\left(#2\right)}

\title{\LARGE \bf
	Distributed Model Predictive Control with Connectivity-based Contracts
}

\author{Jorit Geurts, Danilo Saccani, Melanie N. Zeilinger, Andrea Carron
	\thanks{This research has been supported by the Swiss National Science Foundation under the NCCR Automation (grant 51NF40\_225155).}
	\thanks{J. Geurts, A. Carron and M.N. Zeilinger are with the Institute for Dynamic Systems and Control, ETH Zurich, Switzerland. (email \texttt{\{jgeurts,carrona,mzeilinger\}@ethz.ch)}  }%
	\thanks{D. Saccani is with the Institute of Mechanical Engineering, École Polytechnique Fédérale de Lausanne (EPFL), CH-1015 Lausanne, Switzerland. (email: \texttt{danilo.saccani@epfl.ch})}%
}

\begin{document}

\maketitle
\thispagestyle{empty}
\pagestyle{empty}

\begin{abstract}
	Teams of mobile robots rely on continuous communication with their
	neighbors for coordination, yet most distributed model predictive control
	(DMPC) schemes assume the communication network stays connected rather
	than actively enforcing it. Adding such a guarantee is hard since the
	usual mathematical condition for connectivity is nonconvex and links every
	agent to every other, which is incompatible with a scalable distributed real-time
	controller.
	We propose a DMPC framework in which each agent is assigned a
	connectivity contract: a local region prescribing where its predicted
	positions may lie over the prediction horizon. The contracts are designed
	so that, as long as every agent stays within its own contract, the team
	is guaranteed to remain connected. Given the maintained contract graph, an agent builds its contract from a
	single exchange with its immediate neighbors, after which every agent
	solves its own optimization problem independently. We prove that the resulting closed-loop system maintains
	connectivity, avoids collisions, and respects local state and input
	constraints. Simulation and hardware experiments on miniature autonomous
	car-like robots demonstrate the approach.
\end{abstract}


\section{Introduction}

Multi-agent robotic systems are increasingly explored for tasks that are difficult or inefficient for a single robot, such as search and rescue, cooperative
transportation, environmental monitoring, and exploration. For these systems,
model predictive control (MPC) is an attractive framework because it optimizes
task-dependent objectives while explicitly handling state and input
constraints~\cite{negenborn2009multi}. Centralized MPC, however, scales poorly and introduces a single point of failure. Distributed MPC (DMPC) alleviates these issues by letting agents solve
local optimization problems while exchanging information with nearby
agents~\cite{scattolini2009architectures,muller2017economic,conte2016distributed,boyd2011distributed}.
DMPC relies on a reliable, possibly multi-hop, communication network among the agents. However, most DMPC schemes assume that this network remains connected instead of actively enforcing it~\cite{conte2016distributed}. If connectivity
is lost, information may no longer propagate through the swarm, which can
compromise safety, coordination and controller
performance~\cite{jadbabaie2003coordination,olfati2004consensus,fax2004information}.
Enforcing connectivity inside DMPC is challenging because the connectivity constraint is intrinsically nonconvex and tightly coupled to the relative positions of all agents, making it hard to embed in a scalable, distributed, real-time controller. In this work, we address this gap through a contract-based formulation in which each agent is constrained to a local region built from one-hop information (see Fig.~\ref{fig:concept}).

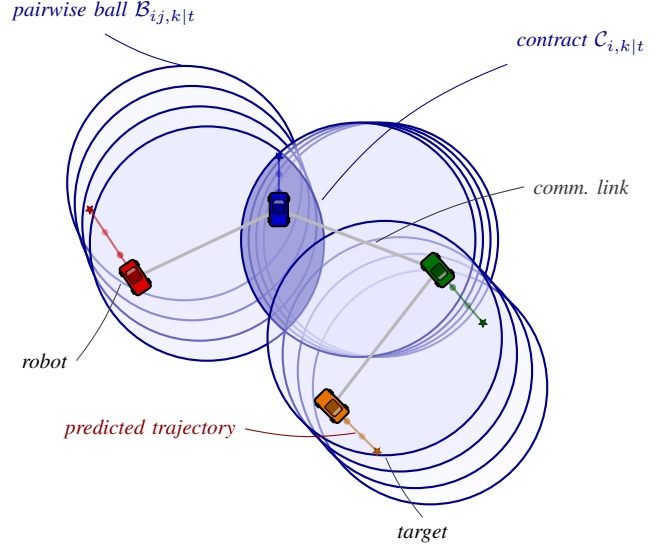
\begin{figure}[!t]
	\centering
	\begin{tikzpicture}[
			scale=1.0,
			every node/.style={font=\footnotesize},
			tree/.style={very thick,gray!55,line cap=round},
			goal/.style={star,star points=5,star point ratio=0.45,fill=#1!75!black,draw=#1!55!black,inner sep=1.1pt,thin},
			goal/.default=blue,
			trajline/.style={thick,#1!75!black,opacity=0.55,line cap=round},
			trajline/.default=blue,
			annot/.style={font=\footnotesize\itshape},
			leader/.style={-,thin,gray!55!black,shorten >=1pt},
		]
		\coordinate (A0) at (-2.2, 0.4);  \coordinate (A1) at (-2.4, 0.7);
		\coordinate (A2) at (-2.6, 1.0);  \coordinate (A3) at (-2.8, 1.3);

		\coordinate (B0) at (-0.3, 1.3);  \coordinate (B1) at (-0.3, 1.53);
		\coordinate (B2) at (-0.3, 1.77); \coordinate (B3) at (-0.3, 2.0);

		\coordinate (C0) at ( 1.8, 0.5);  \coordinate (C1) at ( 2.0, 0.27);
		\coordinate (C2) at ( 2.2, 0.03); \coordinate (C3) at ( 2.4,-0.2);

		\coordinate (D0) at ( 0.4,-1.3);  \coordinate (D1) at ( 0.6,-1.5);
		\coordinate (D2) at ( 0.8,-1.7);  \coordinate (D3) at ( 1.0,-1.9);

		\foreach \k in {0,1,2,3}{
				\coordinate (Mab\k) at ($(A\k)!0.5!(B\k)$);
				\coordinate (Mbc\k) at ($(B\k)!0.5!(C\k)$);
				\coordinate (Mcd\k) at ($(C\k)!0.5!(D\k)$);
			}

		\foreach \k/\op in {3/0.10, 2/0.18, 1/0.28, 0/0.45}{
				\draw[draw=blue!55!black,thick,fill=blue!10,fill opacity=\op] (Mab\k) circle (1.55);
				\draw[draw=blue!55!black,thick,fill=blue!10,fill opacity=\op] (Mbc\k) circle (1.55);
				\draw[draw=blue!55!black,thick,fill=blue!10,fill opacity=\op] (Mcd\k) circle (1.55);
			}

		\begin{scope}
			\clip (Mab0) circle (1.55);
			\clip (Mbc0) circle (1.55);
			\fill[blue!55!black,opacity=0.3] (-5,-5) rectangle (5,5);
		\end{scope}

		\draw[tree] (A0) -- (B0);
		\draw[tree] (B0) -- (C0);
		\draw[tree] (C0) -- (D0);

		\foreach \name/\col in {A/red,B/blue,C/green!45!black,D/orange}{
				\draw[trajline=\col] (\name0) -- (\name1) -- (\name2) -- (\name3);
				\foreach \k/\op in {1/0.55,2/0.35,3/0.2}{
						\fill[\col!75!black,opacity=\op] (\name\k) circle (0.045);
					}
			}

		\node[goal=red]            at (A3) {};
		\node[goal=blue]           at (B3) {};
		\node[goal=green!45!black] at (C3) {};
		\node[goal=orange]         at (D3) {};

		\foreach \name/\col/\ang in {A/red!85!black/124, B/blue!80!black/90, C/green!45!black/-49, D/orange!90!black/-45}{
				\begin{scope}[shift={(\name0)}, rotate=\ang]
					\fill[\col,draw=black,thick,rounded corners=1.2pt] (-0.22,-0.12) rectangle (0.22, 0.12);
					\fill[\col!65!black,draw=black,thin,rounded corners=0.6pt] (-0.06,-0.08) rectangle (0.13, 0.08);
					\fill[white,opacity=0.45] (0.08,-0.07) -- (0.12,-0.05) -- (0.12,0.05) -- (0.08,0.07) -- cycle;
					\fill[black,rounded corners=0.3pt] (-0.18,-0.15) rectangle (-0.10,-0.10);
					\fill[black,rounded corners=0.3pt] (-0.18, 0.10) rectangle (-0.10, 0.15);
					\fill[black,rounded corners=0.3pt] ( 0.10,-0.15) rectangle ( 0.18,-0.10);
					\fill[black,rounded corners=0.3pt] ( 0.10, 0.10) rectangle ( 0.18, 0.15);
				\end{scope}
			}

		\node[annot,blue!55!black] (Lball) at (-2.6, 3.9) {pairwise ball $\mathcal{B}_{ij,k|t}$};
		\draw[leader,blue!55!black] (Lball.south) to[bend right=8] ($(Mab3)+(0.0,1.55)$);

		\node[annot,blue!55!black] (Lcontract) at (3.7, 3.5) {contract $\mathcal{C}_{i,k|t}$};
		\draw[leader,blue!55!black] (Lcontract.south west) to[bend right=15] ($(B0)+(0.55,0.1)$);

		\node[annot,gray!45!black] (Llink) at (3.7, 1.6) {comm.\ link};
		\draw[leader,gray!45!black] (Llink.west) to[bend right=8] ($(Mbc0)+(0.2,-0.05)$);

		\node[annot] (Lrobot) at (-3.4,-0.7) {robot};
		\draw[leader] (Lrobot.east) to[bend right=8] ($(A0)+(-0.18,-0.05)$);

		\node[annot] (Ltarget) at (1.6,-3.0) {target};
		\draw[leader] (Ltarget.north) to[bend left=8] ($(D3)+(0.1,-0.05)$);

		\node[annot,red!55!black] (Ltraj) at (-2.0,-1.6) {predicted trajectory};
		\draw[leader,red!55!black] (Ltraj.east) to[bend right=10] ($(D2)+(-0.05,0.05)$);
	\end{tikzpicture}
	\caption{High-level overview of the proposed contract-based DMPC. Car-like
		robots (top view) exchange predicted trajectories only with one-hop
		neighbors along the communication graph (gray links). For each
		link, the two robots agree to remain in a shared pairwise ball
		of radius $r_{\mathrm{com}}/2$, which keeps them within communication
		range $r_{\mathrm{com}}$. Each robot's contract is the
		intersection of its incident balls (darker region, central robot);
		contracts shrink along the prediction horizon (fading balls), and each
		robot's predicted trajectory (fading dots) stays inside its own
		contract. Every robot then solves its own MPC to reach its target
		(star).}
	\label{fig:concept}
\end{figure}

\subsubsection*{Related work}
Several families of methods have been proposed for connectivity maintenance. Potential-field and gradient-based methods combine collision avoidance and connectivity through local artificial potentials~\cite{zavlanos2007potential,hsieh2008maintaining}. They are distributed but often conservative, and do not naturally accommodate state/input constraints or task-level objectives.

Graph-theoretic methods based on algebraic connectivity maintain the positivity
of the Fiedler eigenvalue $\lambda_2(L)$ of the graph Laplacian $L$.
Since $\lambda_2(L)>0$ if and only if the
communication graph is connected, this quantity provides a natural certificate
for connectivity. Existing methods enforce this condition through gradient-like
controllers, distributed eigenvalue estimation, or passivity-based
strategies~\cite{bullo2018lectures,mohar1991laplacian,degennaro2006decentralized,zavlanos2005controlling,yang2010decentralized,sabattini2012decentralized,zavlanos2011graph,ji2007distributed,giordano2013passivity}.
Compared to potential-field methods, these approaches can be less conservative and have been validated experimentally~\cite{giordano2013passivity}, but they are typically standalone connectivity controllers rather than unified constrained optimization frameworks.

Geometric constraint-set methods replace the eigenvalue condition with local
constraints on relative positions~\cite{ando1999distributed,notarstefano2006maintaining,spanos2004robust}.
Spanning-tree variants reduce conservatism by preserving only a connected subset
of links, but existing schemes are mostly reactive safety mechanisms and do not
directly address constrained optimal control with task-level objectives.

Connectivity constraints have also been embedded directly into MPC. The work in~\cite{carron2023multi} imposes the algebraic-connectivity condition $\lambda_2(L(p))>0$ as an explicit constraint and derives a sequential quadratic programming (SQP) reformulation amenable to distributed implementation. While principled, this introduces nonlinear coupling among agents and leads to a computationally demanding problem.

Contract-based DMPC has recently been used to decouple
collision-avoidance constraints by having agents commit to local sets over the
prediction horizon~\cite{bodmer2026anytime}; however, connectivity has not yet
been treated within this contract framework.

\subsubsection*{Contribution}
This paper proposes a contract-based DMPC framework that enforces
connectivity of the time-varying communication graph while optimizing
task-specific objectives under local state, input, and collision-avoidance
constraints, using only neighbor-to-neighbor communication.
Our contribution is threefold:
\begin{enumerate}
	\item \emph{Connectivity contracts:} convex, time-varying local sets assigned to each agent over the prediction horizon, whose satisfaction preserves a connected spanning subgraph and thus replaces the nonconvex Fiedler-eigenvalue constraint of~\cite{carron2023multi}. The resulting local finite horizon optimal control problem (FHOCP) is fully decoupled and requires no iterative coordination among agents.
	\item \emph{Theoretical guarantees:} validity conditions for the contracts and a proof of recursive feasibility and constraint satisfaction.
	\item \emph{Experimental validation:} simulation and hardware experiments on miniature autonomous car-like robots.
\end{enumerate}

The remainder of the paper is organized as follows. Section~\ref{sec:problem} formulates the connectivity-constrained DMPC problem. Section~\ref{sec:contract_dmpc} introduces the contract-based DMPC framework, presents the construction of the connectivity contracts, and proves recursive feasibility and closed-loop constraint satisfaction. Section~\ref{sec:experiments} reports the simulation and hardware experiments, and Section~\ref{sec:conclusions} concludes the paper.

\section{Problem formulation}\label{sec:problem}

We consider a swarm of agents indexed by
$\mathcal{M}:=\{1,\ldots,M\}$. Each agent $i\in\mathcal{M}$ is described by
\begin{align}
	x_i(t+1) & = f_i(x_i(t),u_i(t)), \label{eq:agent_dyn} \\
	p_i(t)   & = C_i x_i(t), \label{eq:agent_pos}
\end{align}
where $x_i\in\mathbb{R}^{n_i}$, $u_i\in\mathbb{R}^{m_i}$, and
$p_i\in\mathbb{R}^{n_p}$ are the state, input, and position, respectively.
Without loss of generality, we order the state so that the position components
are the first $n_p$ entries; hence $C_i = [I_{n_p}\ \ 0]$.
We make the following assumption on the system dynamics.

\begin{assumption}[Position invariance]\label{ass:posinv}
	For each agent $i\in\mathcal{M}$ and translation
	$\Delta p\in\mathbb{R}^{n_p}$, let
	$\varphi_i(\Delta p):=[\Delta p^\top\ 0^\top]^\top
		\in\mathbb{R}^{n_i}$.
	The dynamics are invariant with respect to translations of the absolute
	position, i.e.,
	\begin{align}
		f_i(x_i+\varphi_i(\Delta p),u_i)
		=
		f_i(x_i,u_i)+\varphi_i(\Delta p),
		\label{eq:position_invariance}
	\end{align}
	for all admissible $x_i$, $u_i$, and $\Delta p$.
\end{assumption}

This assumption holds for many robotic systems in which the position evolves as
the output of an integrator and the dynamics do not depend explicitly on the absolute position~\cite{carron2023multi,bodmer2026anytime, saccani2023model}.

Each agent is subject to local state and input constraints
\begin{align}
	x_i(t)\in\mathcal{X}_i,\qquad
	u_i(t)\in\mathcal{U}_i,
	\qquad \forall i\in\mathcal{M},\ \forall t\in\mathbb{N}.
	\label{eq:local_constraints}
\end{align}
where $\mathcal{X}_i\subseteq\mathbb{R}^{n_i}$ and
$\mathcal{U}_i\subseteq\mathbb{R}^{m_i}$. Agents must also satisfy coupled
collision-avoidance constraints, written as
\begin{align}
	c(p(t))\leq 0,
	\label{eq:collision_constraint}
\end{align}
where $p(t):=\mathrm{col}_{i\in\mathcal{M}}(p_i(t))$ and
$c:\mathbb{R}^{Mn_p}\to\mathbb{R}^{n_c}$ represents inter-agent and
agent-obstacle collision constraints.

\subsection{Communication}

We assume that the agents evolve and communicate synchronously. For a communication
radius $r_{\mathrm{com}}>0$, agents $i$ and $j$ are neighbors at time $t$ if
\begin{align}
	\|p_j(t)-p_i(t)\|_2 \leq r_{\mathrm{com}}.
\end{align}
This induces the time-varying undirected graph
\begin{align}
	\mathcal{G}(p(t))
	=
	\left(\mathcal{M},\mathcal{E}(p(t))\right),
\end{align}
with edge set
\begin{multline}
	\mathcal{E}(p(t))
	:=
	\{
	(i,j)\in\mathcal{M}\times\mathcal{M}\;:\;
	i\neq j, \\ \|p_j(t)-p_i(t)\|_2 \leq r_{\mathrm{com}}
	\}.
\end{multline}
Accordingly, the communication neighbors of agent $i$ are
\begin{align}
	\mathcal{N}^{\mathcal{G}}_i(t)
	:=
	\{j\in\mathcal{M}:\ (i,j)\in\mathcal{E}(p(t))\}.
\end{align}

To characterize swarm connectivity, let $L(p(t))$ denote the Laplacian of $\mathcal{G}(p(t))$. The graph is connected if and only if the second-smallest eigenvalue $\lambda_2(L(p(t)))$ of the Laplacian,  known as the Fiedler eigenvalue, is strictly positive~\cite{fiedler1973algebraic,bullo2018lectures,mohar1991laplacian}. Connectivity of the communication network is therefore enforced by
\begin{align}
	\lambda_2\!\left(L(p(t))\right) > 0,
	\qquad \forall t\in\N. \label{eq:connectivity_constraint}
\end{align}

\subsection{Centralized MPC problem}

The goal is to compute, in a distributed fashion, control inputs that optimize a
task-dependent objective while satisfying local constraints, collision
avoidance, and connectivity. Let $x(t):=\mathrm{col}_{i\in\mathcal{M}}(x_i(t))$,
$u(t):=\mathrm{col}_{i\in\mathcal{M}}(u_i(t))$, and
$p(t):=Cx(t)$, with $C:=\mathrm{diag}(C_1,\ldots,C_M)$. Given a horizon
$N\in\mathbb{N}$, a stage cost $\ell$, and a terminal cost $\ell_f$, the
centralized FHOCP at time $t$ is
\begin{subequations}\label{eq:centralized_fhocp}
	\begin{align}
		\min_{\substack{x_{0|t},\dots,x_{N|t}                                  \\
		u_{0|t},\dots,u_{N-1|t}}}\quad &
		\sum_{k=0}^{N-1}\ell(x_{k|t},u_{k|t})+\ell_f(x_{N|t})
		\label{eq:centralized_fhocp_cost}
		\\
		\mathrm{s.t.}\quad
		                               & x_{i,0|t}=x_i(t),
		\label{eq:centralized_fhocp_init}
		\\
		                               & x_{i,k+1|t}=f_i(x_{i,k|t},u_{i,k|t}),
		\label{eq:centralized_fhocp_dyn}
		\\
		                               & u_{i,k|t}\in\mathcal{U}_i,
		\label{eq:centralized_fhocp_input}
		\\
		                               & x_{i,k|t}\in\mathcal{X}_i,
		\label{eq:centralized_fhocp_state}
		\\
		                               & c(p_{k|t})\leq 0,
		\label{eq:centralized_fhocp_collision}
		\\
		                               & \lambda_2(L(p_{k|t}))>0,
		\label{eq:centralized_fhocp_connectivity}
	\end{align}
\end{subequations}
where $p_{k|t}:=Cx_{k|t}$. Constraints
\eqref{eq:centralized_fhocp_init}--\eqref{eq:centralized_fhocp_state} are
imposed for all $i\in\mathcal{M}$; dynamics and input constraints for
$k\in\mathbb{N}_0^{N-1}$; and state, collision-avoidance, and connectivity
constraints for $k\in\mathbb{N}_0^{N}$.

We assume that the initial configuration is connected, i.e.,
$\lambda_2(L(p(0)))>0$. Problem~\eqref{eq:centralized_fhocp} provides a
centralized MPC formulation of the connectivity-maintenance problem. However,
the collision-avoidance and connectivity constraints couple agents through
their relative positions, and therefore prevent a scalable neighbor-to-neighbor
implementation. The next section replaces these coupled constraints with local
contracts that can be computed and enforced in a distributed manner.

\section{Contract-based distributed MPC} \label{sec:contract_dmpc}

We replace the coupled collision-avoidance and connectivity constraints of
Problem~\eqref{eq:centralized_fhocp} with local contract constraints. A contract
is a set describing where an agent's predicted position is allowed to lie
at a given prediction step. By planning with respect to contracts, rather than
directly against other agents' optimized trajectories, each agent enforces the
coupled requirements through local constraints. This removes
optimization-level coupling between agents.

Validity of the contracts requires two properties: the predicted positions must
remain inside their assigned contracts, and this property must be preserved
under the standard shift of the predicted sequence. We use this contract paradigm for
both connectivity maintenance and collision avoidance.

At time $t$, each agent $i$ is assigned connectivity contracts
\begin{align}
	\mathcal{C}_{i,t}
	:=
	\{\mathcal{C}_{i,0|t},\ldots,\mathcal{C}_{i,N|t}\},
	\qquad
	\mathcal{C}_{i,k|t}\subseteq\mathbb{R}^{n_p},
\end{align}
where $\mathcal{C}_{i,k|t}$ constrains the predicted position of agent $i$ at
prediction step $k$. These contracts will be designed so that their satisfaction
preserves a connected spanning subgraph of the communication graph.

Collision avoidance is also enforced through local contracts
\begin{align}
	\mathcal{D}_{i,t}
	:=
	\{\mathcal{D}_{i,0|t},\ldots,\mathcal{D}_{i,N|t}\},
	\qquad
	\mathcal{D}_{i,k|t}\subseteq\mathbb{R}^{n_p}.
\end{align}
Following~\cite{bodmer2026anytime}, we assume that these contracts satisfy
\begin{align}
	C_i x_{i,k|t}\in \mathcal{D}_{i,k|t},
	\quad \forall i\in\mathcal{M}
	\quad\Longrightarrow\quad
	c(p_{k|t})\leq 0.
	\label{eq:collision_contract_implication}
\end{align}
We refer to~\cite{bodmer2026anytime} for their construction and recursive update
rule. The remainder of this section focuses on the construction of the
connectivity contracts.

\subsection{Connectivity contracts}

We now construct the connectivity contracts replacing the global constraint
\eqref{eq:centralized_fhocp_connectivity}. Instead of imposing the nonconvex
Fiedler-eigenvalue condition directly, we preserve a connected set of
communication links along the prediction horizon. At time $t$, let
\begin{align}
	\mathcal{T}_t
	=
	\left(\mathcal{M},\mathcal{E}_{\mathcal{T}}(t)\right)
\end{align}
be a connected spanning subgraph whose edges are the communication links to be
maintained by the contracts. In the simplest case, $\mathcal{T}_t$ is a spanning tree, so that
only the minimum number of links needed for connectivity is preserved.
The graph $\mathcal{T}_t$ may vary over time. It can be kept fixed or updated at each time step, as long as it remains a connected spanning subgraph of the communication graph; the update is discussed in Section~\ref{subsec:distributed_contr}. We denote by $\mathcal{N}^{\mathcal{T}}_i(t)$ the set of neighbors of agent $i$ in $\mathcal{T}_t$.

The contracts are constructed around a proposed trajectory. For each agent
$i\in\mathcal{M}$, let
\[
	\tilde p_{i,0:N|t}
	:=
	\{\tilde p_{i,0|t},\ldots,\tilde p_{i,N|t}\},
	\qquad
	\tilde p_{i,0|t}=p_i(t),
\]
denote the proposed position trajectory at time $t$. The following property
collects the requirements that this trajectory must satisfy in order to be used
for contract construction.

\begin{definition}[Compatible proposed trajectory]
	\label{def:compatible_proposed_trajectory}
	At time $t$, the proposed position trajectories
	$\{\tilde p_{i,0:N|t}\}_{i\in\mathcal{M}}$ are compatible with the local
	constraints, collision contracts, and contract graph $\mathcal{T}_t$ if, for
	each agent $i\in\mathcal{M}$, there exist trajectories
	$\{\tilde x_{i,k|t}\}_{k=0}^{N}$ and
	$\{\tilde u_{i,k|t}\}_{k=0}^{N-1}$ such that
	\begin{align}
		\tilde x_{i,0|t}
		 & =
		x_i(t),
		\label{eq:proposed_initial_condition}
		\\
		\tilde x_{i,k+1|t}
		 & =
		f_i(\tilde x_{i,k|t},\tilde u_{i,k|t}),\
		\tilde u_{i,k|t}\in\mathcal{U}_i,
		 &     & k\in\mathbb{N}_0^{N-1},
		\label{eq:proposed_dynamics_inputs}
		\\
		\tilde x_{i,k|t}
		 & \in
		\mathcal{X}_i,\quad
		\tilde p_{i,k|t}
		=
		C_i\tilde x_{i,k|t},
		 &     & k\in\mathbb{N}_0^{N}.
		\label{eq:proposed_state_constraint_position}
	\end{align}
	Moreover,
	\begin{align}
		 & \tilde p_{i,k|t}
		\in
		\mathcal{D}_{i,k|t},
		 &                                         & \forall i\in\mathcal{M},\ k\in\mathbb{N}_0^N,
		\label{eq:proposed_inside_collision_contracts}
		\\
		 & \|\tilde p_{i,k|t}-\tilde p_{j,k|t}\|_2
		\leq
		r_{\mathrm{com}},
		 &                                         & \forall (i,j)\in\mathcal{E}_{\mathcal{T}}(t),\
		k\in\mathbb{N}_0^N.
		\label{eq:proposed_trajectory_compatible_tree}
	\end{align}
\end{definition}

Definition~\ref{def:compatible_proposed_trajectory} is only a requirement for
constructing nonempty contracts at a given time $t$. In the closed-loop scheme,
compatible proposed trajectories are generated recursively by the standard MPC
shifting arguments, as shown in Theorem~\ref{thm:recursive_feasibility}. Hence the condition
reduces to the usual initial-feasibility requirement.

We now define
connectivity contracts around this trajectory so that any admissible plan
satisfying them preserves the same links.

\begin{definition}[Connectivity contracts]
	\label{def:connectivity_contracts}
	A family of sets
	\[
		\left\{\mathcal{C}_{i,k|t}\right\}_{i\in\mathcal{M},\,k=0}^{N},
		\qquad
		\mathcal{C}_{i,k|t}\subseteq\mathbb{R}^{n_p},
	\]
	is a valid family of connectivity contracts with respect to $\mathcal{T}_t$ if:

	\begin{enumerate}
		\item the proposed positions are contained in the contracts,
		      \begin{align}
			      \tilde p_{i,k|t}\in\mathcal{C}_{i,k|t},
			      \qquad
			      \forall i\in\mathcal{M},\quad k\in\mathbb{N}_0^N;
			      \label{eq:proposed_positions_inside_connectivity_contracts}
		      \end{align}

		\item for every edge $(i,j)\in\mathcal{E}_{\mathcal{T}}(t)$ and every
		      prediction step $k\in\mathbb{N}_0^N$,
		      \begin{align}
			      \|p_i-p_j\|_2 \leq r_{\mathrm{com}},
			      \quad
			      \forall p_i\in\mathcal{C}_{i,k|t},\
			      \forall p_j\in\mathcal{C}_{j,k|t}.
			      \label{eq:pairwise_connectivity_contract_condition}
		      \end{align}
		      Equivalently,
		      $\mathcal{C}_{i,k|t}\oplus(-\mathcal{C}_{j,k|t})
			      \subseteq \Ball{r_{\mathrm{com}}}{0}$ where $\Ball{r}{c}:=\{p\in\mathbb{R}^{n_p}\mid \|p-c\|_2\le r\}$.
	\end{enumerate}
\end{definition}

The first condition anchors the contracts around the proposed trajectory,
whereas the second condition guarantees that every maintained edge remains
inside the communication range for all admissible positions in the contracts.
We next give a concrete construction satisfying these conditions.

The construction is inspired by classical limited-range connectivity constraints
for mobile agents~\cite{ando1999distributed,notarstefano2006maintaining}, where
a pairwise communication link is preserved by constraining the two agents to
remain inside a ball centered at their midpoint. Here, we adapt this idea to the
finite-horizon setting by constructing such balls around the proposed positions
at each prediction step and only for the edges of the contract graph
$\mathcal{T}_t$. For every edge $(i,j)\in\mathcal{E}_{\mathcal{T}}(t)$ and prediction step
$k\in\mathbb{N}_0^N$, define
\begin{align}
	\mathcal{B}_{ij,k|t}
	:=
	\Ball{r_{\mathrm{com}}/2}{
		\frac{\tilde p_{i,k|t}+\tilde p_{j,k|t}}{2}
	}.
	\label{eq:pairwise_ball_contract}
\end{align}
By~\eqref{eq:proposed_trajectory_compatible_tree}, both proposed positions
$\tilde p_{i,k|t}$ and $\tilde p_{j,k|t}$ belong to
$\mathcal{B}_{ij,k|t}$.
Let
$\mathcal{N}^{\mathcal{T}}_i(t)$ denote the neighbors of agent $i$ in the
contract graph $\mathcal{T}_t$. Agent $i$ then defines
\begin{align}
	\mathcal{C}_{i,k|t}
	:=
	\bigcap_{j\in\mathcal{N}^{\mathcal{T}}_i(t)}
	\mathcal{B}_{ij,k|t},
	\qquad
	k\in\mathbb{N}_0^N.
	\label{eq:connectivity_contract_intersection}
\end{align}
Since each ball $\mathcal{B}_{ij,k|t}$ is convex, every contract
$\mathcal{C}_{i,k|t}$ is convex as an intersection of convex sets.
Figure~\ref{fig:contracts_construction} illustrates this construction for an agent with three tree neighbors.

\begin{figure}[t]
	\centering
	\begin{tikzpicture}[
			scale=0.85,
			every node/.style={font=\footnotesize},
			agent/.style={circle,fill=black,inner sep=1.6pt},
			ball/.style={draw=blue!55!black,thick,fill=blue!10,fill opacity=0.55},
			tree/.style={very thick,gray!60},
			leader/.style={-,thin,gray!55!black,shorten >=1pt},
		]
		\coordinate (Pi) at ( 0.7, 1.3);
		\coordinate (Pj) at (-1.8, 0.7);
		\coordinate (Pk) at ( 2.4,-0.3);
		\coordinate (Pl) at (-0.3,-1.5);

		\coordinate (Mij) at ($(Pi)!0.5!(Pj)$);
		\coordinate (Mik) at ($(Pi)!0.5!(Pk)$);
		\coordinate (Mil) at ($(Pi)!0.5!(Pl)$);

		\draw[ball] (Mij) circle (1.5);
		\draw[ball] (Mik) circle (1.5);
		\draw[ball] (Mil) circle (1.5);

		\begin{scope}
			\clip (Mij) circle (1.5);
			\clip (Mik) circle (1.5);
			\clip (Mil) circle (1.5);
			\fill[blue!55!black,opacity=0.3] (-5,-5) rectangle (5,5);
		\end{scope}

		\draw[tree] (Pi) -- (Pj);
		\draw[tree] (Pi) -- (Pk);
		\draw[tree] (Pi) -- (Pl);

		\foreach \p in {Pi,Pj,Pk,Pl}{\node[agent] at (\p) {};}

		\node[anchor=south west,inner sep=2pt] at (Pi) {$\tilde p_i$};
		\node[anchor=north east,inner sep=2pt] at (Pj) {$\tilde p_j$};
		\node[anchor=west,inner sep=2pt]       at (Pk) {$\tilde p_k$};
		\node[anchor=north west,inner sep=2pt] at (Pl) {$\tilde p_\ell$};

		\node[blue!55!black] (Lbij) at (-2.6, 2.4) {$\mathcal{B}_{ij}$};
		\draw[leader,blue!55!black] (Lbij.south east) to[bend left=10] ($(Mij)+(-0.8,1.0)$);

		\node[blue!55!black] (Lbik) at ( 3.7, 2.3) {$\mathcal{B}_{ik}$};
		\draw[leader,blue!55!black] (Lbik.south west) to[bend right=12] ($(Mik)+(1.05,0.85)$);

		\node[blue!55!black] (Lbil) at ( 2.0,-2.1) {$\mathcal{B}_{i\ell}$};
		\draw[leader,blue!55!black] (Lbil.north west) to[bend right=10] ($(Mil)+(0.95,-0.95)$);

		\node[blue!55!black] (Lci) at (-3.0,-2.0) {$\mathcal{C}_{i}=\bigcap_{n\in\mathcal{N}^{\mathcal{T}}_i}\mathcal{B}_{in}$};
		\draw[leader,blue!55!black] (Lci.north east) to[bend right=20] ($(Pi)+(-0.45,-0.5)$);
	\end{tikzpicture}
	\caption{Construction of the connectivity contracts. For each maintained
		edge of the spanning subgraph $\mathcal{T}_t$ (gray), a pairwise ball
		$\mathcal{B}_{ij}$ of radius $r_{\mathrm{com}}/2$ is centered at the
		midpoint of the proposed positions $\tilde p_i,\tilde p_j$. Agent $i$'s
		contract $\mathcal{C}_{i}$ (darker region) is the intersection of the
		balls of all its incident edges, which guarantees
		$\|p_i-p_j\|\le r_{\mathrm{com}}$ for every admissible predicted position.}
	\label{fig:contracts_construction}
\end{figure}
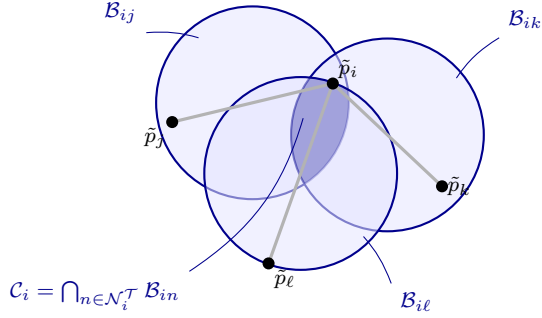

\begin{proposition}[Ball contracts preserve connectivity]
	\label{prop:ball_contracts_preserve_connectivity}
	Let the proposed trajectories be compatible in the sense of
	Definition~\ref{def:compatible_proposed_trajectory}, and let
	$\mathcal{T}_t$ be connected. Then the sets
	$\{\mathcal{C}_{i,k|t}\}_{i\in\mathcal{M},\,k=0}^{N}$ defined by
	\eqref{eq:connectivity_contract_intersection} are valid connectivity contracts
	in the sense of Definition~\ref{def:connectivity_contracts}. Moreover, if
	\begin{align}
		C_i x_{i,k|t}\in\mathcal{C}_{i,k|t},
		\qquad
		\forall i\in\mathcal{M},\quad k\in\mathbb{N}_0^N,
	\end{align}
	then
	\begin{align}
		\lambda_2\!\left(L(p_{k|t})\right)>0,
		\qquad
		k\in\mathbb{N}_0^N.
	\end{align}
\end{proposition}

\begin{proof}
	By Definition~\ref{def:compatible_proposed_trajectory}, for every
	$j\in\mathcal{N}^{\mathcal{T}}_i(t)$, condition
	\eqref{eq:proposed_trajectory_compatible_tree} implies
	$\tilde p_{i,k|t}\in\mathcal{B}_{ij,k|t}$. Therefore
	$\tilde p_{i,k|t}\in\mathcal{C}_{i,k|t}$, and the first condition of
	Definition~\ref{def:connectivity_contracts} holds.

	Moreover, for every maintained edge $(i,j)$ and every
	$k\in\mathbb{N}_0^N$, both $\mathcal{C}_{i,k|t}$ and
	$\mathcal{C}_{j,k|t}$ are subsets of the same ball
	$\mathcal{B}_{ij,k|t}$, whose diameter is $r_{\mathrm{com}}$.
	Therefore, $\|p_i-p_j\|_2\leq r_{\mathrm{com}}$ for all
	$p_i\in\mathcal{C}_{i,k|t}$ and
	$p_j\in\mathcal{C}_{j,k|t}$, which proves
	\eqref{eq:pairwise_connectivity_contract_condition}. Thus the constructed sets
	are valid connectivity contracts.

	If $C_i x_{i,k|t}\in\mathcal{C}_{i,k|t}$ for all agents and prediction steps,
	then every edge of the connected graph $\mathcal{T}_t$ is present in the
	communication graph induced by $p_{k|t}$. Hence this graph is connected, or
	equivalently $\lambda_2(L(p_{k|t}))>0$.
\end{proof}

In implementation, the sets in
\eqref{eq:connectivity_contract_intersection} can be replaced by polytopic inner
approximations containing the proposed positions. This yields linear
position-contract constraints while preserving the connectivity guarantee.

\subsection{Local FHOCP}\label{sec:local_fhocp}

We now define the local optimization problem solved by each agent once the
contracts have been computed. We assume that the stage and terminal cost are additively
decomposable across agents, i.e.,
\begin{align}
	\ell(x,u)=\sum_{i\in\mathcal{M}}\ell_i(x_i,u_i), \quad  \ell_f(x)=\sum_{i\in\mathcal{M}}\ell_{f,i}(x_i),
\end{align}

At time $t$, agent $i$ solves
\begin{subequations}\label{eq:local_contract_fhocp}
	\begin{align}
		\min_{\substack{x_{i,0|t},\ldots,x_{i,N|t}           \\
		u_{i,0|t},\ldots,u_{i,N-1|t}                         \\
		\bar x_{i|t},\bar u_{i|t}}}
		\quad &
		\sum_{k=0}^{N-1}\ell_i(x_{i,k|t},u_{i,k|t})+\ell_{f,i}(\bar{x}_{i|t})
		\label{eq:local_contract_fhocp_cost}
		\\
		\mathrm{s.t.}\quad
		      & x_{i,0|t}=x_i(t),
		\label{eq:local_contract_fhocp_init}
		\\
		      & x_{i,k+1|t}=f_i(x_{i,k|t},u_{i,k|t}),
		\label{eq:local_contract_fhocp_dyn}
		\\
		      & x_{i,k|t}\in\mathcal{X}_i,
		\label{eq:local_contract_fhocp_state}
		\\
		      & u_{i,k|t}\in\mathcal{U}_i,
		\label{eq:local_contract_fhocp_input}
		\\
		      & C_i x_{i,k|t}\in\mathcal{D}_{i,k|t},
		\label{eq:local_contract_fhocp_collision_contract}
		\\
		      & C_i x_{i,k|t}\in\mathcal{C}_{i,k|t},
		\label{eq:local_contract_fhocp_connectivity_contract}
		\\
		      & x_{i,N|t}=\bar x_{i|t},
		\label{eq:local_contract_fhocp_terminal_eq}
		\\
		      & \bar x_{i|t}=f_i(\bar x_{i|t},\bar u_{i|t}),
		\label{eq:local_contract_fhocp_terminal_ss}
		\\
		      & \bar x_{i|t}\in\mathcal{X}_i,\quad
		\bar u_{i|t}\in\mathcal{U}_i.
		\label{eq:local_contract_fhocp_terminal_admissible}
	\end{align}
\end{subequations}
The dynamics~\eqref{eq:local_contract_fhocp_dyn} and input
constraints~\eqref{eq:local_contract_fhocp_input} are imposed for all
$k\in\mathbb{N}_0^{N-1}$, while the state constraints
\eqref{eq:local_contract_fhocp_state}, collision contracts
\eqref{eq:local_contract_fhocp_collision_contract}, and connectivity contracts
\eqref{eq:local_contract_fhocp_connectivity_contract} are imposed for all
$k\in\mathbb{N}_0^{N}$.

The last three constraints impose a terminal equality to an admissible steady
state selected by the optimizer. By Assumption~\ref{ass:posinv}, the terminal
equilibrium can be translated in position without changing the required
steady-state input. Hence, the optimizer can choose a terminal position
compatible with the terminal collision and connectivity contracts~\eqref{eq:local_contract_fhocp_collision_contract}--\eqref{eq:local_contract_fhocp_connectivity_contract}.

This terminal equality yields a simple recursive-feasibility argument: once the
terminal state is reached, applying the corresponding terminal input keeps the
agent at the same admissible state. More general terminal ingredients could be
used, but they would have to remain compatible with the time-varying contracts
and their recursive update.

With these ingredients, problem~\eqref{eq:local_contract_fhocp} is fully local:
the coupled collision-avoidance and connectivity constraints of the centralized
problem are replaced by the set-membership constraints
\eqref{eq:local_contract_fhocp_collision_contract} and
\eqref{eq:local_contract_fhocp_connectivity_contract}. Therefore, once the
contracts have been computed, no optimization variable of another agent appears
in~\eqref{eq:local_contract_fhocp}, and each agent can solve its FHOCP
independently.


\subsection{Theoretical guarantees}\label{subsec:guarantees}

We now state the recursive-feasibility and constraint-satisfaction result for
the proposed contract-based MPC scheme.
\begin{theorem}
	\label{thm:recursive_feasibility}
	Suppose that the local FHOCPs~\eqref{eq:local_contract_fhocp} are feasible at
	$t=0$, and that the initial proposed trajectories are compatible in the sense of
	Definition~\ref{def:compatible_proposed_trajectory}. At each time $t$, assume
	that the collision avoidance contracts satisfy
	\eqref{eq:collision_contract_implication} and are recursively updated so as to
	contain the shifted proposed positions, as in~\cite{bodmer2026anytime}.
	Moreover, let the connectivity contracts be constructed according to
	\eqref{eq:pairwise_ball_contract}--\eqref{eq:connectivity_contract_intersection}.

	Then the local FHOCPs~\eqref{eq:local_contract_fhocp} are recursively feasible.
	Moreover, the closed-loop trajectories satisfy the local state and input
	constraints~\eqref{eq:local_constraints}, the collision-avoidance
	constraint~\eqref{eq:collision_constraint}, and the connectivity
	constraint~\eqref{eq:connectivity_constraint} for all $t\in\mathbb{N}$.
	Lastly there always exists a connected subgraph $\mathcal{T}_t$ for all $t\in\mathbb{N}$.
\end{theorem}

\begin{proof}
	Assume that the local FHOCPs are feasible at time $t$, and let
	$\{x^\star_{i,k|t}\}_{k=0}^{N}$,
	$\{u^\star_{i,k|t}\}_{k=0}^{N-1}$, and
	$(\bar x^\star_{i|t},\bar u^\star_{i|t})$ denote the optimal solution of
	\eqref{eq:local_contract_fhocp}. We construct the candidate proposed trajectory
	at time $t+1$ by the standard MPC shift:
	\begin{align}
		\tilde x_{i,k|t+1}
		 & :=
		x^\star_{i,k+1|t},
		 &    & k\in\mathbb{N}_0^{N-1},
		\label{eq:shifted_state_1}
		\\
		\tilde x_{i,N|t+1}
		 & :=
		\bar x^\star_{i|t}
		=
		x^\star_{i,N|t},
		\label{eq:shifted_state_2}
		\\
		\tilde u_{i,k|t+1}
		 & :=
		u^\star_{i,k+1|t},
		 &    & k\in\mathbb{N}_0^{N-2},
		\label{eq:shifted_input_1}
		\\
		\tilde u_{i,N-1|t+1}
		 & :=
		\bar u^\star_{i|t}.
		\label{eq:shifted_input_2}
	\end{align}
	After applying $u_i(t)=u^\star_{i,0|t}$, the shifted candidate satisfies
	$\tilde x_{i,0|t+1}=x^\star_{i,1|t}=x_i(t+1)$. The dynamics, state constraints,
	and input constraints follow from feasibility at time $t$ and from the terminal
	steady-state and admissibility constraints
	\eqref{eq:local_contract_fhocp_terminal_eq}--\eqref{eq:local_contract_fhocp_terminal_admissible}.

	By the recursive collision-contract construction of~\cite{bodmer2026anytime},
	the contracts at time $t+1$ can be chosen so that the shifted proposed
	positions satisfy
	\[
		C_i\tilde x_{i,k|t+1}\in\mathcal{D}_{i,k|t+1},
		\qquad
		\forall i\in\mathcal{M},\quad k\in\mathbb{N}_0^N.
	\]
	Thus the shifted trajectory satisfies the collision-contract part of
	Definition~\ref{def:compatible_proposed_trajectory}.

	It remains to verify that a connected subgraph $\mathcal{T}_{t+1}$ exists and that the shifted trajectory is compatible with it.
	Since the solution at time $t$ satisfies the connectivity contracts, and these contracts
	are valid by Proposition~\ref{prop:ball_contracts_preserve_connectivity}, every
	maintained edge satisfies
	\[
		\|p^\star_{i,k|t}-p^\star_{j,k|t}\|_2
		\leq r_{\mathrm{com}},
		\qquad
		\forall (i,j)\in\mathcal{E}_{\mathcal{T}}(t),\quad
		k\in\mathbb{N}_0^N,
	\]
	where $p^\star_{i,k|t}:=C_i x^\star_{i,k|t}$. Therefore, the shifted positions
	satisfy the same bound at time $t+1$ for all
	$k\in\mathbb{N}_0^N$; for $k=N$, this follows from the repeated terminal
	steady state. Hence every edge of $\mathcal{T}_t$ is preserved under the shift, so we can always set $\mathcal{T}_{t+1}=\mathcal{T}_t$, showing the existence of a connected subgraph. Compatibility of the shifted trajectory with $\mathcal{T}_{t+1}$ simply follows from compatibility of the shifted solution with $\mathcal{T}_t$.


	The ball construction
	\eqref{eq:pairwise_ball_contract}--\eqref{eq:connectivity_contract_intersection}
	therefore provides valid connectivity contracts containing the shifted proposed
	trajectory. Consequently, the shifted candidate is feasible for the local
	FHOCPs at time $t+1$. Recursive feasibility follows by induction from
	feasibility at $t=0$.

	Finally, since the applied input is always the first input of a feasible
	solution, the local state and input constraints hold in closed loop. Collision
	avoidance follows from~\eqref{eq:collision_contract_implication}, and
	connectivity follows from Proposition~\ref{prop:ball_contracts_preserve_connectivity}.
\end{proof}

\subsection{Distributed contract computation} \label{subsec:distributed_contr}

The contract construction requires only one-hop communication along the
maintained contract graph. Each agent $i$ stores its incident tree neighbors
$\mathcal{N}^{\mathcal{T}}_i(t)$ and, at the beginning of each MPC step,
exchanges the proposed position trajectory
\[
	\tilde p_{i,0:N|t}
	:=
	\{\tilde p_{i,0|t},\ldots,\tilde p_{i,N|t}\}
\]
with all $j\in\mathcal{N}^{\mathcal{T}}_i(t)$. For each maintained edge
$(i,j)\in\mathcal{E}_{\mathcal{T}}(t)$, the two agents construct the same
pairwise ball~\eqref{eq:pairwise_ball_contract}. Agent $i$ then obtains its
connectivity contracts by intersecting the balls associated with its incident
edges, as in~\eqref{eq:connectivity_contract_intersection}. Thus, once
$\mathcal{T}_t$ is given, no multi-hop consensus is required to compute the
contracts.

For recursive feasibility, the contract graph can be initialized as any
connected spanning subgraph of the initial communication graph and then retained
recursively, $\mathcal{T}_{t+1}=\mathcal{T}_t$. Optional graph updates may reduce
conservatism. To preserve the recursive-feasibility argument, any updated graph
must be chosen as a connected spanning subgraph whose edges are compatible with
the proposed trajectories. In the numerical experiments of Section~\ref{sec:experiments}, this update is
implemented using a distributed spanning-tree construction
\cite{gallager1983distributed,pandurangan2017time}.


The implementation is synchronous and executed at discrete control instants:
within each control period, agents exchange the trajectory messages, construct
the contracts, solve their local FHOCPs, and apply the first input.
In experiments, the communication
radius used in the contract construction can be tightened from
$r_{\mathrm{com}}$ to $r_{\mathrm{com}}-\delta$, where $\delta>0$ acts as a
robustness buffer for bounded communication delays, state-estimation errors,
and model mismatch. Larger delays that prevent the required trajectory messages
from being available within the control period would instead require an
explicit asynchronous or delay-robust extension, which is outside the scope of
the present formulation.

The resulting closed-loop procedure is summarized in
Algorithm~\ref{alg:contract_dmpc}.

\begin{algorithm}[t]
	\caption{Contract-based distributed MPC executed by agent $i$}
	\label{alg:contract_dmpc}
	\begin{algorithmic}[1]
		\State Initialize a feasible proposed position trajectory $\tilde p_{i,0:N|0}$ and a connected graph $\mathcal{T}_0$.
		\For{$t=0,1,2,\ldots$}
		\State Exchange $\tilde p_{i,0:N|t}$ with all
		$j\in\mathcal{N}^{\mathcal{T}}_i(t)$.
		\State Compute $\mathcal{C}_{i,k|t}$ using
		\eqref{eq:pairwise_ball_contract}--\eqref{eq:connectivity_contract_intersection}.
		\State Compute collision contracts $\mathcal{D}_{i,k|t}$ according to~\cite{bodmer2026anytime}.
		\State Solve the local FHOCP~\eqref{eq:local_contract_fhocp}.
		\State Apply $u_i(t)=u^\star_{i,0|t}$.
		\State Set $\mathcal{T}_{t+1}=\mathcal{T}_t$ and build
		$\tilde p_{i,0:N|t+1}$ from the shifted state trajectory \eqref{eq:shifted_state_1}--\eqref{eq:shifted_state_2}.
		\EndFor
	\end{algorithmic}
\end{algorithm}

\section{Experimental validation}
\label{sec:experiments}

We validate the proposed contract-based DMPC scheme in simulation and on miniature autonomous vehicles. The experiments demonstrate that the proposed connectivity contracts maintain network connectivity while the agents track reference positions in cluttered environments and satisfy collision-avoidance constraints.

\subsection{Experimental setup}

The implementation is based on a modular ROS-based software framework for miniature car-like robots~\cite{carron2023chronos} and deployed on 1/28-scale vehicles. Each vehicle is modeled by a discrete-time kinematic bicycle model with state and input
\begin{equation*}
	x_i=[p_{i,x},p_{i,y},\psi_i,v_i]^\top,\quad
	u_i=[\delta_i,a_i]^\top,
\end{equation*}
where $(p_{i,x},p_{i,y})$ is the position, $\psi_i$ the heading angle, $v_i$ the velocity, $\delta_i$ the steering angle, and $a_i$ the longitudinal acceleration. In simulation, the model is integrated using a fourth-order Runge--Kutta scheme. In hardware, poses are obtained from a motion-capture system\footnote[1]{\url{https://www.qualisys.com/cameras/arqus/}} and processed through the state-estimation pipeline of the same framework.

Each agent solves its local FHOCP in C++ using acados~\cite{verschueren2022acados} with an SQP-type nonlinear programming solver and HPIPM~\cite{frison2020hpipm} as QP solver. The prediction horizon is $N=20$, and the control period is $T_{\mathrm{ctrl}}=40\,\mathrm{ms}$.
The communication radius is $r_{\mathrm{com}}=1.25\,\mathrm{m}$. To account for model mismatch, state-estimation errors, and communication delays, the connectivity contracts are built with the tightened radius $r_{\mathrm{com}}-\delta$, where $\delta=r_{\mathrm{buffer}}=0.05\,\mathrm{m}$.

The connectivity contracts are implemented as polytopic inner approximations of the balls in~\eqref{eq:pairwise_ball_contract}. We use regular polygons with $R=8$ vertices, which was found to provide a good tradeoff between approximation quality and computational complexity. Collision-avoidance contracts are constructed as in~\cite{bodmer2026anytime}. Communication is implemented through ROS publishers and subscribers. Each agent exchanges its proposed position trajectory with the neighbors in the maintained contract graph and discards messages from agents outside the communication range. Artificial communication delays are introduced by allowing messages to be exchanged every $T_{\mathrm{com}}=5\,\mathrm{ms}$.
In the experiments, we update the spanning subgraph $\mathcal{T}_t$ at each time $t$ by having the agents construct the minimum spanning tree of the network with Kruskal's algorithm in a distributed manner.

\subsection{Simulation results}

We first evaluate the method in simulation on a reference-tracking task with $M=7$ agents. The agents are placed in four cluttered environments and are assigned sampled reference positions. In total, we consider 31 initial/reference configurations. The agents must reach their references while satisfying collision-avoidance constraints and maintaining connectivity of the communication graph.

We compare three methods: a contract-based DMPC controller using collision-avoidance contracts only~\cite{bodmer2026anytime}, the proposed controller, and the eigenvalue-constrained MPC of~\cite{carron2023multi}. The latter enforces $\lambda_2(L(p_{k|t}))\ge\underline{\lambda}_2=0.1$ at every prediction step on a Laplacian with smooth distance-dependent weights and solves the resulting coupled nonconvex program by SQP; the weights are chosen such that the constraint certifies connectivity of the communication graph. We report it both with the SQP iterated to convergence and with a single real-time iteration (RTI) per control period, as used by the proposed scheme. Figure~\ref{fig:sim_traj} shows a representative run of the proposed controller. Without connectivity contracts, the agents may move toward their references in a way that disconnects the communication graph. In contrast, the proposed method preserves the maintained spanning subgraph throughout the maneuver while still allowing the agents to progress toward their references.

\begin{figure}[t]
	\centering
	\includegraphics[width=0.99\columnwidth]{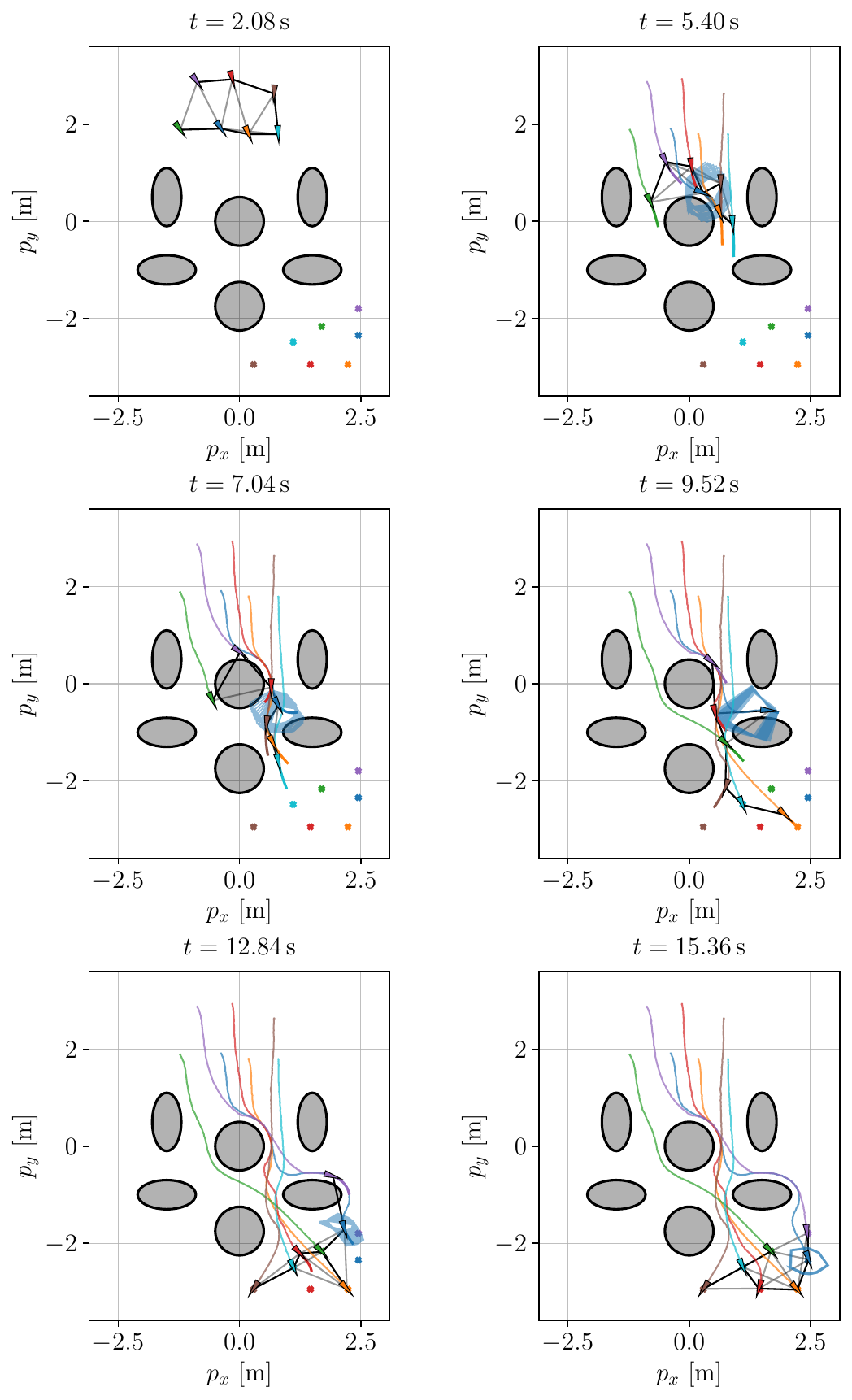}
	\caption{Simulation example with \(M=7\) agents in a cluttered environment. Colored triangles denote the agents and colored squares their reference positions. The proposed connectivity contracts preserve a connected spanning
		subgraph while the agents track their references. One agent shows the resulting
		connectivity contracts.}
	\label{fig:sim_traj}
\end{figure}

The closed-loop behavior is summarized in Table~\ref{tab:sim_metrics}. We report the number of connectivity violations, the minimum algebraic connectivity $\min_t\lambda_2(L(p(t)))$, the minimum collision distance, and the per-agent solve time statistics.
The eigenvalue-constrained MPC also maintains connectivity when its SQP is iterated to convergence, but at a substantially higher computational cost, since it solves a coupled nonconvex program over all agents. With a single real-time iteration its solve time becomes comparable to that of the local problems of the proposed scheme, but one Newton step on the linearized eigenvalue and collision constraints does not track the solution of the nonconvex problem: the communication graph disconnects in 4 runs and the safety distance is violated in 6 runs.

\begin{table}[t]
	\centering
	\caption{Simulation metrics over 31 runs in four cluttered environments.}
	\label{tab:sim_metrics}
	\setlength{\tabcolsep}{2pt}
	\footnotesize
	\begin{tabular}{@{}lcccc@{}}
		\toprule
		                             & Without conn.                      & Proposed & \multicolumn{2}{c}{Fiedler MPC~\cite{carron2023multi}}           \\
		\noalign{\global\aboverulesep=0.1ex \global\belowrulesep=0.15ex}
		\cmidrule(lr){4-5}
		\noalign{\global\aboverulesep=0.4ex \global\belowrulesep=0.65ex}
		Metric                       & contracts~\cite{bodmer2026anytime} &          & SQP                                                    & SQP-RTI \\
		\midrule
		Runs with conn. violations   & 31                                 & 0        & 0                                                      & 4       \\
		Minimum $\lambda_2(L(p(t)))$ & 0.0                                & 0.1981   & 0.1981                                                 & 0.0     \\
		Min. collision distance [m]  & 0.1341                             & 0.1175   & 0.1200                                                 & 0.0806  \\
		Median solve time [ms]       & 2.02                               & 2.31     & 219                                                    & 19      \\
		95th perc. solve time [ms]   & 3.77                               & 4.29     & 7661                                                   & 50      \\
		\bottomrule
	\end{tabular}
\end{table}



\subsection{Hardware experiment}

The same controller is deployed on the car-like robots equipped with an ESP32 microcontroller and Wi-Fi communication. Motion-capture markers are used for pose tracking, and the states are estimated using an extended Kalman filter.

The hardware experiment uses $M=7$ vehicles and the same reference position tracking tasks in a cluttered environment. The hardware experiments are conducted with and without connectivity contracts and repeated a total of five times.

%

Figure~\ref{fig:hw_traj} shows the resulting hardware trajectories when using connectivity contracts. The vehicles track their references while avoiding obstacles and maintaining connectivity.
Figure~\ref{fig:hw_metrics} plots the corresponding algebraic connectivity for the experiment with and without the connectivity contracts. We can see that without the connectivity contracts $\lambda_2$ drops to zero indicating a disconnected network. In contrast using the proposed connectivity contracts the system remains connected ($\lambda_2 > 0$). The closed-loop behavior is summarized in Table~\ref{tab:hardware_metrics}.

\begin{figure}[h]
	\centering
	\includegraphics[width=\columnwidth]{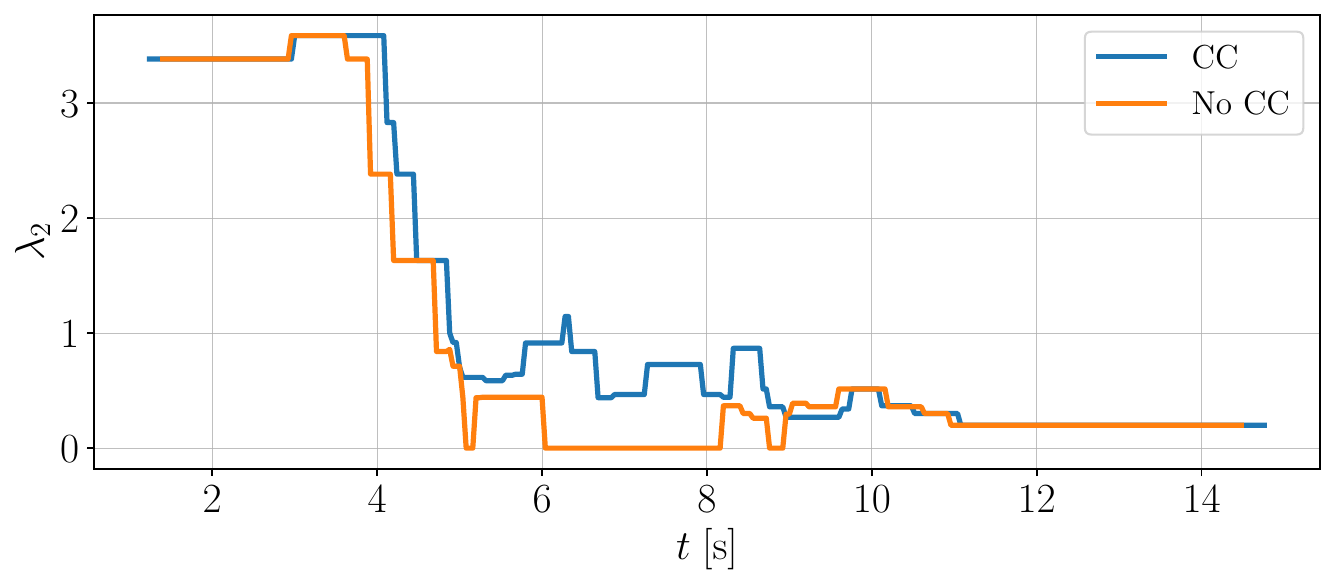}
	\caption{Algebraic connectivity $\lambda_2(L(p(t)))$ over time for one of the hardware experiments with and without connectivity contracts~(CC).}
	\label{fig:hw_metrics}
\end{figure}

\begin{figure*}[t]
	\centering
	\begin{minipage}{0.49\textwidth}
		\includegraphics[width=\linewidth, trim={0 858.6pt 0 0}, clip]{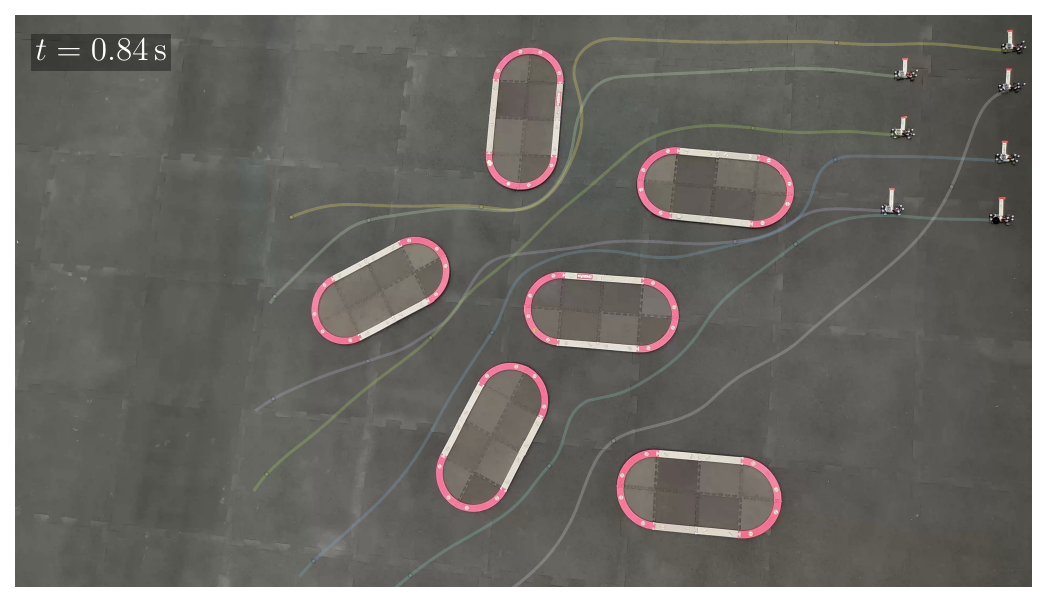}
	\end{minipage}\hfill
	\begin{minipage}{0.49\textwidth}
		\includegraphics[width=\linewidth, trim={0 572.4pt 0 286.2pt}, clip]{figures/real_1_trajectories_overlay.pdf}
	\end{minipage}\\[2pt]
	\begin{minipage}{0.49\textwidth}
		\includegraphics[width=\linewidth, trim={0 286.2pt 0 572.4pt}, clip]{figures/real_1_trajectories_overlay.pdf}
	\end{minipage}\hfill
	\begin{minipage}{0.49\textwidth}
		\includegraphics[width=\linewidth, trim={0 0 0 858.6pt}, clip]{figures/real_1_trajectories_overlay.pdf}
	\end{minipage}
	\caption{Hardware experiment with 1/28-scale miniature vehicles. The agents
		track their references while satisfying collision-avoidance and
		connectivity contracts.}
	\label{fig:hw_traj}
\end{figure*}

\begin{table}[t]
	\centering
	\caption{Hardware metrics.}
	\label{tab:hardware_metrics}
	\setlength{\tabcolsep}{4pt}
	\begin{tabular}{lcc}
		\toprule
		Metric                          & Without conn. contracts & Proposed \\
		\midrule
		Minimum $\lambda_2(L(p(t)))$    & 0.0                     & 0.198    \\
		Minimum collision distance [m]  & 0.139                   & 0.132    \\
		Median solve time [ms]          & 2.26                    & 3.04     \\
		95th percentile solve time [ms] & 3.64                    & 5.00     \\
		\bottomrule
	\end{tabular}
\end{table}

\section{Conclusions}\label{sec:conclusions}
We presented a contract-based DMPC framework that enforces connectivity of
the time-varying communication graph through convex, neighbor-local
constraints, replacing the nonconvex Fiedler-eigenvalue condition with
pairwise ball contracts built from a single one-hop exchange. The resulting
local FHOCPs are fully decoupled and admit recursive feasibility, collision
avoidance, and connectivity guarantees. Simulation and hardware experiments
on miniature car-like robots confirmed the practicality of the approach.
Future work will address heterogeneous communication radii, asynchronous and
delay-robust contract updates, and adaptive selection of the contract graph
and contract shapes to reduce conservatism.

\bibliographystyle{IEEEtran}
\bibliography{biblio}

\end{document}